%% file: smp_main.tex
\documentclass{article}

\PassOptionsToPackage{numbers, compress}{natbib}
\usepackage[table, xcdraw]{xcolor}
 \usepackage[final]{neurips_2020}

\usepackage[utf8]{inputenc} 
\usepackage[T1]{fontenc}    
\usepackage{hyperref}       
\usepackage{url}            
\usepackage{booktabs}       
\usepackage{amsfonts}       
\usepackage{nicefrac}       
\usepackage{microtype}      
\input{math_commands.tex}

\usepackage[ruled,vlined]{algorithm2e}
\usepackage{tikz}
\usepackage{amsthm}
\usepackage{enumitem}
\usepackage{subcaption}
\usepackage[toc,page]{appendix}
\usepackage{dsfont}
\graphicspath{{figures/}{../../figures/}}

\usepackage{wrapfig}
\definecolor{airforceblue}{rgb}{0.36, 0.54, 0.66}
\definecolor{red}{rgb}{0.7, 0.01, 0.1}

\newtheorem{lemma}{Lemma}
\newtheorem{theorem}{Theorem}
\newtheorem{corollary}{Corollary}
\newtheorem{proposition}{Proposition}

\newcommand{\diam}{\Delta}
\newcommand{\dmax}{d_{\text{max}}}
\newcommand{\davg}{d_{\text{avg}}}

\title{Building powerful and equivariant graph neural networks with structural message-passing}

\author{%
 Clément Vignac, Andreas Loukas, and Pascal Frossard \\
  EPFL\\
  Lausanne, Switzerland\\
  \texttt{\{clement.vignac,andreas.loukas,pascal.frossard\}@epfl.ch} \\
}

\begin{document}

\maketitle

\begin{abstract}
Message-passing has proved to be an effective way to design graph neural networks, as it is able to leverage both permutation equivariance and an inductive bias towards learning local structures in order to achieve good generalization. However, current message-passing architectures have a limited representation power and fail to learn basic topological properties of graphs. We address this problem and propose a powerful and equivariant message-passing framework based on two ideas: first, we propagate a one-hot encoding of the nodes, in addition to the features, in order to learn a \emph{local context} matrix around each node. This matrix contains rich local information about both features and topology and can eventually be pooled to build node representations. Second, we propose methods for the parametrization of the message and update functions that ensure permutation equivariance. Having a representation that is independent of the specific choice of the one-hot encoding permits inductive reasoning and leads to better generalization properties. Experimentally, our model can predict various graph topological properties on synthetic data more accurately than previous methods and achieves state-of-the-art results on molecular graph regression on the ZINC dataset.
\end{abstract}

\section{Introduction}

Graph neural networks have recently emerged as a popular way to process and analyze graph-structured data. Among the numerous architectures that have been proposed, the class of message-passing neural networks (MPNNs) \citep{scarselli2008graph,gilmer2017neural,battaglia2018relational} has been by far the most widely adopted. 
In addition to being able to efficiently exploit the sparsity of graphs, MPNNs exhibit an inherent tendency to learn relationships between nearby nodes. This inductive bias is generally considered as a good fit for problems that require relational reasoning~\citep{Xu2020What}, such as tractable relational inference~\citep{yoon2019inference,satorras2019combining}, problems in combinatorial optimization~\citep{khalil2017learning,li2018combinatorial,joshi2019efficient,karalias2020erdos} or the simulation of physical interactions between objects~\citep{battaglia2016interaction,sanchez2018graph}.

A second key factor to the success of MPNNs is their equivariance properties. Since neural networks can ultimately only process tensors, in order to use a graph as input, it is necessary to order its nodes and build an adjacency list or matrix. Non-equivariant networks tend to exhibit poor sample efficiency as they need to explicitly learn that all representations of a graph in the (enormous) symmetry group of possible orderings actually correspond to the same object. On the contrary, permutation equivariant networks, such as MPNNs, are better equipped to generalize as they already implement the prior knowledge that any ordering is arbitrary.

Despite their success, equivariant MPNNs possess limited expressive power~\citep{xu2018powerful,morris2019weisfeiler}.  
For example, they cannot learn whether a graph is connected, what is the local clustering coefficient of a node, or if a given pattern such as a cycle is present in a graph~\citep{chen2020can}. For tasks where the graph structure is important, such as the prediction of chemical properties of molecules~\citep{elton2019deep, sun2019graph} and the solution to combinatorial optimization problems, more powerful graph neural networks are necessary.


Aiming to address these challenges, this work puts forth \textit{structural message-passing} (SMP)---a new type of graph neural network that is strictly more powerful than MPNNs, while also sharing the attractive inductive bias of message-passing architectures.
SMP inherits its power from its ability to manipulate node identifiers. However, in contrast to previous studies that relied on identifiers~\citep{murphy2019relational,Loukas2020What}, it does so \emph{in a permutation equivariant way} without introducing new sources of randomness. 
As a result, SMP can be powerful without sacrificing its ability to generalize to unseen data. In particular, we show that if SMP is built out of powerful layers, the resulting model is computationally universal over the space of equivariant functions.

Concretely, SMP maintains at each node a matrix called ``local context'' (instead of a feature vector as in MPNNs) that is initialized with a one-hot encoding of the nodes and the node features. These local contexts are then propagated in such a way that a permutation of the nodes or a change in the one-hot encoding will reorder the lines of each context without changing their content, which is key to efficient learning and good generalization.


We evaluate SMP on a diverse set of structural tasks that are known to be difficult for message-passing architectures, such as cycle detection, connectivity testing, diameter and shortest path distance computation. In all cases, our approach compares favorably to previous methods: 
for example, SMP solves cycle detection in all evaluated configurations, whereas other powerful networks struggle when the graphs become larger, and MPNNs do not manage to solve the task completely. 
%

Finally, we evaluate our method on the ZINC chemistry dataset and achieve state-of-the-art performance among methods that do not use expert features. It shows that SMP is able to successfully learn both from the features and from topological information, which is essential in chemistry applications.
Overall, our method is able to overcome a major limitation of MPNNs, while retaining their ability to process features with a bias towards locality.

\paragraph{Notation.}
In the following, we consider the problem of representation learning on one or several graphs of possibly varying sizes. Each graph $G = (V, E)$ has an adjacency matrix $\mA \in \R^{n \times n}$, and potentially node attributes $\mX = (\vx_1, ..., \vx_n)^T \in \R^{n \times c_X}$ and edge attributes $\vy_{ij} \in \R^{c_Y}$ for every $(v_i,v_j) \in E$. These attributes are aggregated into a 3-d tensor $\tY \in \R^{n \times n \times c_Y}$. We consider the edge weights of weighted graphs as edge attributes and view $\mA$ as a binary adjacency matrix. The set of neighbors of a node $v_i \in V$ is written as $N_i$.

\section{Related work}

\subsection{Permutation equivariant graph neural networks}

Originally introduced by \citet{scarselli2008graph}, MPPNs have progressively been extended to handle edge \citep{gilmer2017neural} and graph-level attributes \citep{battaglia2018relational}. Despite the flexibility in their parametrization, MPNNs without special node attributes however all have limited expressive power, even in the limit of infinite depth and width.
For instance, they are at most as good at isomorphism testing as the Weisfeiler-Lehman (WL) vertex refinment algorithm~\citep{weisfeiler1968reduction, xu2018powerful}. The WL test has higher dimensional counterparts ($k$-WL) of increasing power, which has motivated the introduction of the more powerful $k$-WL networks~\citep{morris2019weisfeiler}. However, these higher-order networks are global, in the sense that they iteratively update the state of a $k$-tuple of nodes based on all other nodes (and not only neighbours), a procedure which is very costly both in time and memory. While a faster procedure was proposed in \cite{morris2020weisfeiler} concurrently to our work, key differences with SMP remain: we propose to learn richer embeddings for each node instead of one embedding per k-tuple of nodes, and build our theoretical analysis on distributed algorithms rather than vertex refinement methods.



Recent studies have also characterized the expressive power of MPNNs from other perspectives, such as the ability to approximate  continuous functions on graphs \citep{chen2019equivalence} and solutions to combinatorial problems~\citep{sato2019approximation}, highlighting similar limitations of MPNNs --- see also~ \citep{barcelo2019logical,geerts2020let,Sato2020ASO,garg2020generalization,magner2020power}.
%


Beyond higher-order message-passing architectures, there have been efforts to construct more powerful equivariant networks. One way to do so is to incorporate hand-crafted topological features (such as the presence of cliques or cycles) \cite{bouritsas2020improving}, which requires expert knowledge on what features are relevant for a given task. A more task-agnostic alternative is to build networks by arranging together a set of simple permutation equivariant functions and operators. These building blocks are:
\begin{itemize}[noitemsep, nolistsep]
	\item Linear equivariant functions between tensors of arbitrary orders: a basis for these functions was computed by \citet{maron2018invariant}, by solving the linear system imposed by equivariance.
	\item Element-wise functions, applied independently to each feature of a tensor.
	\item Operators that preserve equivariance, such as $+$, $-$, tensor and elementwise products, composition and concatenation along the dimension of the channels.
	
\end{itemize}
Similarly to \citet{morris2019weisfeiler}, networks built this way obtain a better expressive power than MPNN by using higher-order tensors \citep{kondor2018covariant, maron2018invariant}. Since $k$-th order tensors can represent any $k$-tuple of nodes, architectures manipulating them can exploit more information to compute topological properties (and be as powerful as the $k$-WL test).
Unfortunately, memory requirements are exponential in the tensor order, which makes these methods of little practical interest. More recently, \citet{maron2019provably} proposed provably powerful graph networks (PPGN) based on the observation that the use of matrix multiplication can make their model more expressive for the same tensor order. This principle was also used in the design of Ring-GNN \citep{chen2019equivalence}, which has many similarities with PPGN. Key differences between such methods and ours are that (\textit{i}) SMP can be parametrized to have a lower time complexity, due to the ability of message-passing to exploit the sparsity of adjacency matrices, (\textit{ii})
SMP retains the message-passing inductive bias, which is different from PPGN and, as we will show empirically, makes it better suited to practical tasks such as the detection of substructures in a graph. 

\subsection{Non-equivariant graph neural networks}

In order to better understand the limitations of current graph neural networks, analogies with graph theory and distributed systems have been exploited. In these fields, a large class of problems cannot be solved without using node identifiers \citep{alon1995color, suomela2013survey}. The reasoning is that, in message-passing architectures, each node has access to a local view of the graph created by the reception of messages.
Without identifiers, each node can count the number of incoming messages and process them, but cannot tell from how many unique nodes they come from. They are therefore unable to reconstruct the graph structure.


This observation has motivated researchers to provide nodes with randomly selected identifiers~\citep{murphy2019relational,Loukas2020What,dasoulas2019coloring, sato2020random}. Encouragingly, by showing the equivalence between message-passing and a model in distributed algorithms, \citet{Loukas2020What} proved that graph neural networks with identifiers and sufficiently expressive message and update functions can be Turing universal, 
which was also confirmed on small instances of the graph isomorphism problem~\citep{loukas2020hard}.

 
Nevertheless, the main issue with these approaches is sample efficiency. Identifiers introduce a dependency of the network on a random input and the loss of permutation equivariance, causing poor generalization. Although empirical evidence has been presented that the aforementioned dependency can be overcome with large amounts of training data or other augmentations~\citep{murphy2019relational,loukas2020hard}, overfitting and optimization issues can occur.
In this work, we propose to overcome this problem by introducing a network which is both powerful and permutation equivariant.

\section{Structural message-passing} \label{section:SMP}

We present the structural message-passing neural networks (SMP), as generalization of MPNNs that follows a similar design principle. However, rather than processing vectors with permutation invariant operators, SMP propagates \textit{matrices} and processes them in a permutation equivariant way. This subtle change greatly improves the network's ability of to learn information about the graph structure.

\subsection{Method}

In SMP, each node of a graph maintains a \textit{local context} matrix $\bm{U}_i \in \R^{n \times c}$ rather than a feature vector $\bm{x}_i \in \R^{c}$ as in MPNN.
The $j$-th row of $\mU_i$ contains the $c$-dimensional representation that node $v_i$ has of node $v_j$.
Intuitively, equivariance means that the lines of the local context after each layer are simply permuted when the nodes are reordered, as shown in Fig. \ref{fig:smp}.

\begin{figure}
	\centering
	\includegraphics[width=0.75\textwidth]{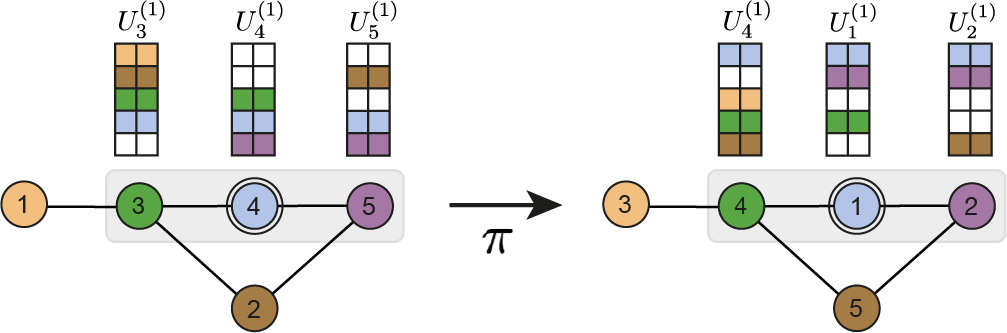}
	\caption{In the SMP model, each local context $\mU_i^{(l)}$ is an $n \times c_l$ matrix, with each row storing the $c_l$-dimensional representation of a node (denoted by color).
	The figure shows the local context in the output of the first layer and blank rows correspond to nodes that have not been encountered yet. Upon node reordering, the lines of the local context are permuted but their content remains unchanged.	\vspace{-4mm}} \label{fig:smp}
\end{figure}

\paragraph{Initialization} The local context is initialized as a one-hot encoding $\bm{U}_i^{(0)}= \bm{\mathds{1}}_i \in \R^{n\times 1}$ for every $v_i \in V$, which corresponds to having initially a unique identifier for each node. In addition, if there are features $\vx_i$ associated with node $v_i$, they are appended to the same row of the local context as the identifiers: $\mU_i^{(0)}[i, :] = [1, \vx_i] \in \R^{1 + c_X}$. Later, we will show that when SMP is parametrized in the proper way, the ordering induced by the one-hot encoding is actually irrelevant to the output.  

%
%

\paragraph{Layers} At layer $l+ 1$, the state of each node is updated as in standard MPNNs \citep{battaglia2018relational}: messages are computed on each edge before being aggregated into a single matrix via a symmetric function. The result  can then be updated using the local context of previous layer at this node:
\begin{align*}
\bm{U}_i^{(l+1)} = u^{(l)}\hspace{-2px}\left(\mU_i^{(l)},  \tilde{\mU}_i^{(l)} \right) \in \mathbb{R}^{n \times c_{l+1}} \quad \text{with} \quad \tilde{\mU}_i^{(l)} = \phi\left(\left\{m^{(l)}(\bm{U}_i^{(l)}, \bm{U}_j^{(l)}, \bm{y}_{ij} ) \right\}_{v_j \in N_i}\right)
\end{align*}
Above, $u^{(l)}$, $m^{(l)}$, $\phi$ are the \textit{update}, \textit{message} and \textit{aggregation} functions of the $(l+1)$-th layer, respectively, whereas $c_{l+1}$ denotes the layer's width.

It might be interesting to observe that, starting from a one-hot encoding and using the update rule $\mU_i^{(l+1)} = \sum_{v_j \in N_i} \mU_j^{(l)}$,  SMP iteratively compute powers of $\mA$. Since $\mA^l[i,j]$ corresponds to the count of walks of length $l$ between $v_i$ and $v_j$, there is a natural connection between the propagation of identifiers and the detection of topological features: even with simple parametrizations, SMP can manipulate polynomials in the adjacency matrix and therefore learn spectral properties \citep{sandryhaila2014big} that MPNNs cannot detect. 

In the following, it will be convenient to express each SMP layer $f^{(l)}$ in a tensor form:
$$
\tU^{(l+1)} = f^{(l)}(\tU^{(l)}, \tY, \mA) = [\bm{U}_1^{(l+1)}, \ldots, \bm{U}_n^{(l+1)}] \ \in \ \mathbb{R}^{n \times n \times c_{l + 1}} 
$$
%

\paragraph{Pooling} After all $L$ message-passing layers have been applied, the aggregated contexts $\tU^{(L)}$ can be pooled to a vector or to a matrix (e.g, for graph and node classification, respectively). To obtain an equivariant representation for node classification, we aggregate each $\mU_i^{(L)} \in \R^{n \times c_L}$ into a vector using an equivariant neural network for sets $\sigma$ \citep{zaheer2017deep, qi2017pointnet, lee2019set, segol2019universal} applied simultaneously to each node $v_i$: 
\[
f_{\textit{eq}} (\tU^{(0)}, \tY, \mA) = \sigma \circ f^{(L)} \circ \cdots \circ f^{(1)}(\tU^{(0)}, \tY, \mA) \ \in \ \R^{n \times c},
\]
whereas a permutation invariant representation is obtained after the application of a pooling function $\textit{pool}$. It may be a simple sum or average followed by a soft-max, or a more complex operator \citep{ying2018hierarchical}: 
\[
f_{\textit{inv}}(\tU^{(0)}, \tY, \mA) = \textit{pool} \circ f_{\textit{eq}} (\tU^{(0)}, \tY, \mA) \ \in \ \R^c 
\]
\subsection{Analysis}

The following section characterizes the equivariance properties and representation power of SMP. For the sake of clarity, we defer all proofs to the appendix. 

\paragraph{Equivariance} 
Before providing sufficient conditions for permutation equivariance, we define it formally. A change in the ordering of $n$ nodes can be described by a permutation $\pi$ of the symmetric group $\mathfrak{S}_n$. $\pi$ acts on a tensor by permuting the axes indexing nodes (but not the other axes): 
\[ 
    (\pi ~.~ \tU)[i, j, k] = \tU[\pi^{-1}(i),~ \pi^{-1}(j), k], \quad \text{where} \quad \tU \in \R^{n \times n \times c}
\]
For vector and matrices, the action of a permutation is more easily described using the matrix $\mPi$ canonically associated to $\pi$: $\pi . z = z $ for $z \in \R^c$, $\pi. \mX = \mPi^T \mX$ for $\mX \in \R^{n \times c}$, and $\pi . \mA = \mPi^T \mA ~\mPi$ for $\mA \in \R^{n \times n}$. 
An SMP layer $f$ is said to be permutation equivariant if permuting the inputs and applying $f$ is equivalent to first applying $f$ and then permuting the result:
\[
 \forall \pi \in \mathfrak{S}_n, \quad \pi ~.~  f(\tU, \tY, \mA) =  f( \pi .\tU, \pi. \tY, \pi. \mA))
 \]
 
We stress that an equivariant SMP network should yield the same output (up to a permutation) for every one-hot encoding used to construct the node identifiers. We can now state some sufficient conditions for equivariance:
\begin{theorem}[Permutation equivariance] \label{theorem:equivariance}
Let functions $m$, $\phi$ and $u$ be permutation equivariant, that is, for every permutation $\pi \in \mathfrak S_n$ we have $u({\pi} . \mU, \pi . \mU') = {\pi} .  u(\mU, \mU') $, $\phi(\{ \pi . \mU_j \}_{v_j \in N_i}) = \pi. \phi(\{ \mU_j \}_{v_j \in N_i})$, and 
$m(\pi . \mU, \pi. \mU', \vy) = {\pi}.m(\mU, \mU', \vy)$. Then, SMP is permutation equivariant. 
\end{theorem}

The proof is presented in Appendix \ref{proof:theorem1}. This theorem defines the class of functions that can be used in our model. For example, if the message and update functions are operators applied simultaneously to each row of the local context, the whole layer is guaranteed to be equivariant. However, more general functions can be used: each $\mU_i$ is a $n \times c$ matrix which can be viewed as the representation of a set of nodes. Hence, any equivariant neural network for sets can be used, which allows the network to have several desirable properties:
\begin{itemize}[noitemsep, nolistsep]
    \item \textit{Inductivity}: as an equivariant neural network for sets can take sets of different size as input, SMP can be trained on graphs with various sizes as well. Furthermore, it can be used in inductive settings on graphs whose size has not been seen during training, which we will confirm experimentally.
    \item \textit{Invariance to local isomorphisms}: SMP learns \textit{structural embeddings}, in the sense that it yields the same result on isomorphic subgraphs. More precisely, if the subgraphs $G_i^k$ and $G_j^k$ induced by $G$ on the k-hop neighborhoods of $v_i$ and $v_j$ are isomorphic, then on node classification, any $k$-layer SMP $f$  will yield the same result for $v_i$ and $v_j$. This  is in contrast with several popular methods \citep{zhang2018link, you2019position} that learn \textit{positional embeddings} which do not have this property.
\end{itemize}

\paragraph{Representation and expressive power} 
The following theorem characterizes the representation power of SMP when parametrized with powerful layers. Simply put, Theorem~\ref{theorem:representation_power_informal} asserts that it is possible to parameterize an SMP network such that it maps non-isomorphic graphs to different representations:

\begin{theorem}[Representation power -- informal]
Consider the class $S$ of simple graphs with $n$ nodes, diameter at most $\diam$ and degree at most $\dmax$. 
We assume that these graphs have respectively $c_X$ and $c_Y$ attributes on the nodes and the edges. Then, there exists a SMP network $f$ of depth at  most $\diam + 1$ and width at most $2 \dmax + c_X + n ~c_Y$ such that the full structure of any graph in $S$ (with the attributes) can be recovered from the output of $f$ at any node.

\label{theorem:representation_power_informal}
\end{theorem}

The formal statement and the proof are detailed in Appendix \ref{appendix:proof-turing}. We first show the result for the simple case where $f$ can pass messages of size $n \times n $, and then consider the case of $n \times 2 \dmax$ matrices using 
the following lemma:
\begin{lemma}[\citet{maehara1990dimension}] For any simple graph $G = (V, E)$ of $n$ nodes and maximum degree $\dmax$, there exists a unit-norm embedding of the nodes $\mX \in \R^{n \times 2 \dmax}$ such that for every $v_i,v_j \in V,~ (v_i, v_j) \in E \iff \langle\mX_i, \mX_j\rangle = 0$ .
\label{lemma:embedding}
\end{lemma} 
The universality of SMP is a direct corollary: since each node can have the ability to reconstruct the adjacency matrix from its local context, it can also employ a universal network for sets~\citep{zaheer2017deep} to compute any equivariant function on the graph (cf. Appendix \ref{proof: universality}). Interestingly, this result shows that propagating matrices instead of vectors might be a way to solve the bottleneck problem \citep{alon2020bottleneck}: while MPNNs need feature maps that exponentially grow with the graph size in order to recover the topology, SMPs can do it with $O(\dmax n^2)$ memory.

\begin{corollary}[Expressive power]
Let $G$ be a simple graph of diameter at most $\diam$ and degree at most $\dmax$. Consider an SMP $f = f^{(L)} \circ \cdots \circ f^{(1)}$ of depth $L =\diam$ and width $2 \dmax + c_X + n ~c_Y $ satisfying the properties of Theorem \ref{theorem:representation_power}. Then, any equivariant function can be computed as $f_\textit{eq} = \sigma \circ f$, where $\sigma$ is a universal function of sets 
applied simultaneously to each node. Similarly, any permutation invariant function can be computed as $f_\textit{in} = \frac{1}{n} \sum_{v_i \in V} \sigma \circ f$.
\label{cor:universality}
\end{corollary}

These results show that two components are required to build a universal approximator of functions on graphs. First, one needs an algorithm that breaks symmetry during message passing, which SMP manages to do in an equivariant manner. Second, one needs powerful layers to parametrize the message, aggregation and update functions. Here, we note that the proofs of Theorem~\ref{theorem:representation_power_informal} and Corollary~\ref{cor:universality} are not constructive and that deriving practical parametrizations that are universal remains an open question \citep{keriven2019universal}. Nevertheless, we do constructively prove the following more straightforward claim using a simple parametrization:

\begin{proposition} \label{prop:strictly}
SMP is strictly more powerful than MPNN: SMP can simulate any MPNN with the same number of layers, but MPNNs cannot simulate all SMPs. 
\end{proposition}
To prove it, we create for any MPNN a corresponding SMP which performs the same operations as the MPNN on the main diagonal of the local context. On the contrary, we can easily create an SMP which is able to distinguish between two small graphs that cannot be distinguished by the Weisfeiler-Lehman test (Appendix \ref{proof-prop1}).

\section{Implementation} 
\label{section:parametrization}

SMP offers a lot of flexibility in its implementation, as any equivariant function that combines the local context of two nodes and the edge features can be used. We propose two implementations that we found to work well, but our framework can also be implemented differently. In both cases, we split the computation of the messages in two steps. First, the local context of each node is updated using a neural network for sets. Then, a standard message passing network is applied separately on each row of the local contexts.
For the first step, we use a subset of the linear equivariant functions computed by \citet{maron2018invariant}:
\[
\forall v_i \in V, \quad
\hat{\mU}_i^{(l)} = \mU_i^{(l)} ~ \bm{W}_1^{(l)} + \frac{1}{n} \bm{1}_n ~  \bm{1}_n^T  ~ \mU_i^{(l)} ~ \bm{W}_2^{(l)} + \bm{1}_n (\vc^{(l)})^\top + \frac{1}{n} \bm{\mathds{1}}_i \bm{1}^T \mU_i^{(l)}~ \bm{W}_3^{(l)}, 
\]
where $\bm{1}_n \in \R^{n \times 1}$ is a vector of ones, $\bm{\mathds 1}_i \in \R^{n \times 1}$ the indicator of $v_i$, whereas $(\mW_k^{(l)})_{1 \leq k \leq 5}$ and $\vc^{(l)}$ are learnable parameters. As for the message passing architecture, we propose two implementations with different computational complexities:
\vspace{-0.2cm}
\paragraph{Default SMP} This architecture corresponds to a standard MPNN, where the message and update functions are two-layer perceptrons. We use a sum aggregator normalized by the average degree $\davg$ over the graph: it retains the good properties of the sum aggregator~\citep{xu2018powerful}, while also avoiding the exploding-norm problem~\citep{velickovic2020neural}. This network can be written:
\begin{equation} 
m_\text{def}^{(l)}(\hat \mU_i^{(l)}, \hat \mU_j^{(l)}, \vy_{ij}) = \textit{MLP}(\hat{\mU}_i^{(l)}, \hat{\mU}_j^{(l)}, \vy_{ij})
\label{eq:default}
\end{equation}
\[ \textstyle
    \mU_i^{(l+1)} = \textit{MLP}(\hat\mU_i^{(l)}, \sum_{v_j \in N_i}  m_\text{def}^{(l)}(\mU_i^{(l)}, \mU_j^{(l)}, \vy_{ij}) / \davg),
\]

\vspace{-0.2cm}
\paragraph{Fast SMP} For graphs without edge features, we propose a second implementation with a message function that uses a pointwise multiplication $\odot$:
$$
m_\text{fast}^{(l)}(\hat \mU_i^{(l)}, \hat \mU_j^{(l)}) = \hat{\mU}_j^{(l)} + \left( \hat{\mU}_i^{(l)} \mW_4^{(l)} \right) \odot \left( \hat{\mU}_j^{(l)} \mW_5^{(l)} \right),
$$
where $\mW_4$ and $\mW_5$ are learnable matrices. The aggregation is the same, and the update is simply a residual connection, so that the $l$-th SMP layer updates each node's local context as
\begin{align*}
\mU_i^{(l+1)} 
&= \textstyle \hat\mU_i^{(l)} + \frac{1}{\davg} \sum_{v_j \in N_i} m_\text{fast}^{(l)}(\mU_i^{(l)}, \mU_j^{(l)}) \\
&= \textstyle \hat\mU_i^{(l)} +  \left(\sum_{v_j \in N_i} \hat{\mU}^{(l)}_j ~+~  \hat\mU^{(l)}_i \mW_4^{(l)} \odot \sum_{v_j \in N_i} \hat\mU^{(l)}_j \mW_5^{(l)} \right) / \davg
\label{eq:fast}
\end{align*}

In this last equation, the arguments of the two sums are only functions of the local context of node $v_j$.
This allows for a more efficient implementation, where one message is computed per node, instead of one per edge as in default SMP. 

One might notice that Fast SMP can be seen as a local version of PPGN (proof in Appendix \ref{proof:prop2}):

\begin{proposition} 
\label{prop:2} A Fast SMP with $k$ layers can be approximated by a $2k$-block PPGN.
\end{proposition}
Despite not being more powerful, Fast SMP has the advantage of being more efficient than PPGN, as it can exploit the sparsity of adjacency matrices. Furthermore, as we will see experimentally, our method manages to learn topological information much more easily than PPGN, a property that we attribute to the inductive bias carried by message-passing.

\begin{table}[t]
\centering
	\caption{Time and space complexity of the forward pass expressed in terms of number of nodes $n$, number of edges $m$,  number of node colors $\chi$, and width $c$. For connected graphs, we trivially have $\chi \leq n \leq m + 1  \leq n^2$.\vspace{1.5mm}}
\resizebox{0.73\textwidth}{!}{
	\begin{tabular}{@{}l l l} \toprule
		Method & Memory per layer & Time complexity per layer \\
		\midrule
		GIN \citep{xu2018powerful}& $\Theta(n ~ c)$ & $\Theta(m ~ c + n ~c^2)$ \\
		MPNN \citep{gilmer2017neural}& $\Theta(n ~ c)$ & $\Theta(m ~c^2)$ \\ \midrule
		Fast SMP (with coloring) & $\Theta(n ~\chi ~ c)$ & $\Theta(m ~\chi ~ c + n ~\chi ~ c^2) $ \\
		Fast SMP & $\Theta(n^2~ c)$ & $\Theta(m ~ n ~ c + n^2 ~ c^2) $ \\
		SMP & $\Theta(n^2~c)$ & $\Theta(m ~n ~ c^2) $ \\
		PPGN \citep{maron2019provably}& $\Theta(n^2~ c)$ & $\Theta(n^3~c + n^2~c^2) $ \\
		Local order-3 WL \citep{morris2019weisfeiler}& $\Theta(n^3~ c) $ & $\Theta(n^4~c + n^3~c^2)$ \\
		\bottomrule \\
	\end{tabular}
	}
	\vspace{-0.4cm}
 \label{complexity}
 
\end{table}

\paragraph{Complexity}  Table \ref{complexity} compares the per-layer space and time complexity induced by the forward pass of SMP with that of other standard graph networks. Whereas local order-$3$ Weisfeiler-Lehman networks need to store all triplets of nodes, both PPGN and SMP only store information for pairs of nodes. However, message-passing architectures (such as SMP) can leverage the sparsity of the adjacency matrix and hence benefit from a more favorable time complexity than architectures which perform global updates (as PPGN).
 
An apparent drawback of SMP (shared by all equivariant powerful architectures we are aware of) is the need for more memory than MPNN. This difference is partially misleading since it is known that the width of any MPNN needs to grow at least linearly with $n$ (for any constant depth) for it to be able to solve many graph-theoretic problems~\citep{Loukas2020What,corso2020principal,loukas2020hard}. However, for graphs with a large diameter, the memory requirements of SMP can be relaxed by using the following observation: \emph{if each node is colored differently from all nodes in its $2k$-hop neighborhood, then no node will see the same color twice in its $k$-hop neighborhood}. It implies that nodes which are far apart can use the same identifier without conflict. We propose in Appendix \ref{appendix:coloring} a procedure (Fast SMP with coloring) based on greedy coloring which can replace the initial one-hot encoding, so that each node can manipulate smaller matrices $\mU_i$. This method allows to theoretically improve both the time and space complexity of SMP, although the number of colors needed usually grows fast with the number of layers in the network.

\section{Experiments} \label{section:Experiments}


\subsection{Cycle detection}

\begin{table}[t!]
    \caption{Experiments on cycle detection, viewed as a graph classification problem.}
    \vspace{-0.2cm}
\input{cycle_results}
    \begin{subtable}{\textwidth}
    \centering
    \caption{Test accuracy (\%) when evaluating the generalization ability of inductive networks. Each network is trained on one graph size (``In-distribution''), validated on a second size, then tested on a third (``Out-of-distribution''). SMP is the only powerful network evaluated that generalizes well. \textit{OOM} = out of memory.}
    \vspace{-0.1cm}
    \input{table_generalization}
    \label{tab:generalization}
    \end{subtable}
    
    \begin{subtable}{\textwidth}
    \caption{Test accuracy (\%) on the detection of 6 cycles for graphs with 56 nodes trained on less data. Thanks to its equivariance properties, SMP requires much less data for training.}
    \vspace{-0.1cm}
    \centering
    \begin{tabular}{@{} l c c c c @{}} \toprule
    Train samples & $200$&$500$&$1000$&$5000$ \\ 
    \midrule
    GIN + random identifiers&$65.8$ &$70.8$&$80.6$&$96.4$ \\
    SMP & $\textbf{87.7}$& $\textbf{97.4}$ & $\textbf{97.6}$ & $\textbf{99.5}$ \\ \bottomrule 
    \end{tabular}
    \label{tab:less-samples}
    \vspace{-.1cm}
    \end{subtable}
    \label{tab:cycle_results}
\end{table}

\begin{figure}
    \centering
		\includegraphics[width=\textwidth]{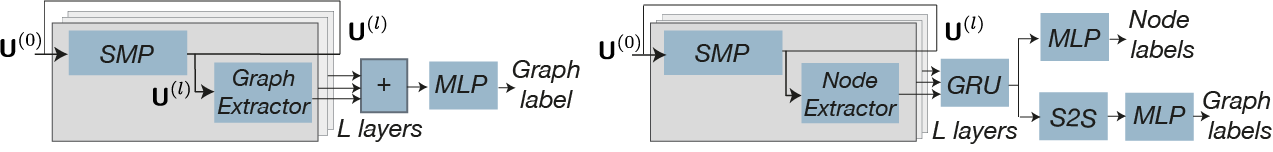}
    \caption{(left) Architecture for cycle detection. The graph extractor computes the trace and the sum along the two first axes of $\tU$, and passes the result into a two-layer MLP in order to produce a set of global features. (right) Architecture for multi-task learning: 
    after each convolution, node features are extracted using a two-layer MLP followed by three pooling methods (mean, max, and the extraction of $\tU[i, i, :]$ for each $v_i \in V$), and a final linear layer. The rest of the architecture is similar to~\citet{corso2020principal}: it uses a Gated Recurrent Unit (GRU) and a Set-to-set network (S2S). \vspace{-2mm}}
    \label{fig:architecture}
\end{figure}

We first evaluate different architectures on the detection of cycles of length 4, 6 and 8 (for several graph sizes), implemented as a graph classification problem\footnote{Our implentation with Pytorch Geometric~\citep{fey2019fast} is available at \url{github.com/cvignac/SMP}.}. Models are retrained for each cycle length and graph size on $10$k samples with balanced classes, and evaluated on $10,000$ samples as well. 
The same architecture (detailed in Figure \ref{fig:architecture}) is used for all models, as we found it to perform better than the original implementation of each method: the methods under comparison thus only differ in the definition of the convolution, making comparison easy. 
We use the fast implementation of SMP, as we find its expressivity to be sufficient for this task.


Results are shown in Tab \ref{tab:cycle_results}. For a given cycle length, the task becomes harder as the number of nodes in the graph grows: the bigger the graph, the more candidate paths that the network needs to verify as being cycles. SMP is able to solve the task almost perfectly for all graph and cycle sizes. For standard message-passing models, we observe a correlation between accuracy and the presence of identifiers: random identifiers and weak identifiers (a one-hot encoding of the degree) tend to perform better than the baseline GIN and MPNN.
PPGN and RING-GNN solve the task well for small graphs, but fail when $n$ grows.
Perhaps due to a miss-aligned inductive bias, we encountered difficulties with training them, whereas message-passing architectures could be trained more easily. We provide a more detailed comparison between SMP, PPGN and Ring-GNN in Appendix~\ref{appendix:provably-smp}. We also compare the generalization ability of the different networks that can be used in inductive settings. GIN generalizes well, but SMP is the only one that achieves good performance among the powerful networks. This may be imputable to the inductive bias of message passing architectures, shared by GIN and SMP. 
Finally, we compare SMP and GIN with random identifiers in settings with less training data: SMP requires much fewer samples to achieve good performance, which confirms that equivariance is important for good sample efficiency.



\subsection{Multi-task detection of graph properties}

We further benchmark SMP on the multi-task detection of graph properties proposed in~\citet{corso2020principal}. The goal is to estimate three node-defined targets: \textit{geodesic distance} from a given node (Dist.), \textit{node eccentricity} (Ecc.), and computation of \textit{Laplacian features} $\mL \vx$ given a vector $\vx$ (Lap.), as well as three graph-defined targets: \textit{connectivity} (Conn.), \textit{graph diameter} (Diam.), and \textit{spectral radius} (Rad.). The training set is composed of 5120 graphs with up to 24 nodes, while graphs in the test set have up to 19 nodes. Several MPNNs are evaluated as well as PNA \citep{corso2020principal}, a message-passing model based on the combination of several aggregators. Importantly, random identifiers are used for all these models, so that all baseline methods are theoretically poweful \citep{Loukas2020What}, but not equivariant.

\begin{table}[t]
\vspace{-0.3cm}
\caption{Log MSE on the test set (lower is better). Baseline results are from \citet{corso2020principal}.}
    \centering
\resizebox{0.84\textwidth}{!}{
    \begin{tabular}{@{}l c c c c c c c@{}}
    \toprule
        Model & \textit{Average} & Dist.& Ecc. & Lap. & Conn. & Diam. & Rad. \\
        \midrule 
GIN      	&$-1.99$&$-2.00$	&$-1.90$	&$-1.60$	&$-1.61$	&$-2.17$	&$-2.66$\\
GAT	        &$-2.26$&$-2.34$	&$-2.09$	&$-1.60$	&$-2.44$	&$-2.40$	&$-2.70$\\
MPNN (sum)	&$-2.53$	&$-2.36$	&$-2.16$	&$-2.59$	&$-2.54$	&$-2.67$	&$-2.87$\\
PNA	        &$-3.13$	&$-2.89$	&$-2.89$	&$-3.77$	&$-2.61$	&$-3.04$	&$-3.57$ \\ \midrule 
Fast MPNN (Ablation) & $-2.37$ & $-2.47$ & $-1.99$ & $-2.83$ & $-1.61$ & $-2.40$ & $-2.93$ \\
MPNN (Ablation)	& $-2.77$		&$-3.18$&$-2.05$&$-3.27$&$-2.24$&$-2.88$&$-2.97$\vspace{0.1cm}\\
\bf{Fast SMP} &$-3.53$&$-3.31$&$-3.36$&$\bm{-4.30}$&$-2.72$&$\bm{-3.65}$&$\bm{-3.82}$ \\ 
\bf{SMP} &$\bm{-3.59}$ &$\bm{- ~3.59}$&$\bm{-3.67}$&$-4.27$&$\bm{-2.97}$&$-3.58$&$-3.46$\\
\bottomrule\\
\end{tabular} 
}
\label{tab:multi-task}
\vspace{-0.3cm}
\end{table}

All models are benchmarked using the same architecture, apart from the fact that SMP manipulates local contexts. In order to pool these contexts into node features and use them as input to the Gated Recurrent Unit \citep{cho2014learning}, we use an extractor described in Figure \ref{fig:architecture}. As an ablation study, we also consider for each model a corresponding MPNN with the same architecture.

The results are summarized in Table~\ref{tab:multi-task}. 
We find that both SMPs are able to exploit the local contexts, as they perform much better than the corresponding MPNN.
SMP also outperforms other methods by a significant margin. 
Lastly, standard SMP tends to achieve better results than fast SMP on tasks that require graph traversals (shortest path computations, excentricity, checking connectivity), which may be due to a better representation power.

\subsection{Constrained solubility regression on ZINC}

\begin{table}[t]
    \centering
    \caption{Mean absolute error (MAE) on ZINC, trained on a subset of 10k molecules.}
    \begin{tabular}{lrr}
    \toprule
        Model & No edge features & With edge features \\
        \hline
        Gated-GCN \cite{fey2020hierarchical} & $0.435$& $0.282$\\
        GIN \cite{fey2020hierarchical}& $0.408$ & $0.252$\\
        PNA \cite{corso2020principal}& $0.320$ & $0.188$ \\
        DGN \cite{beaini2020directional}& $\bm{0.219}$ & $0.168$ \\
        MPNN-JT \cite{fey2020hierarchical}& -- & $0.151$ \\
        \hline
        MPNN (ablation) & $0.272$ & $0.189$ \\
        \textbf{SMP (Ours)} & $\bm{0.219}$ & $\bm{0.138}$ \\
        \bottomrule
    \end{tabular}
\vspace{-0.2cm}
    \label{tab:zinc}
\end{table}

The ZINC database is a large scale dataset containing molecules with up to 37 atoms. The task is to predict the constrained solubility of each molecule, which can be seen as a graph regression problem. We follow the setting of \cite{dwivedi2020benchmarking}: we train SMP on the same subset of 10,000 molecules with a parameter budget of around 100k, reduce the learning rate when validation loss stagnates, and stop training when it reaches a predefined value. We use an expressive parametrization of SMP, with 12 layers and 2-layer MLPs both in the message and the update functions. In order to reduce the number of parameters, we share the same feature extractor after each layer (cf Fig. \ref{fig:architecture}). Results are presented in Table \ref{tab:zinc}. They show that in both cases (with or without edge features, which are a one-hot encoding of the bond type), SMP is able to achieve state of the art performance. Note however than even better results (0.108 MAE using a MPNN with edge features \citep{bouritsas2020improving}) can be achieved by augmenting the input with expert features. We did not use them in order to compare fairly with the baseline results.

\section{Conclusion}

We introduced structural message-passing (SMP), a new architecture that is both powerful and permutation equivariant, solving a major weakness of previous message-passing networks. Empirically, SMP significantly outperforms previous models in learning graph topological properties, but retains the inductive bias of MPNNs and their good ability to process node features. 
We believe that our work paves the way to graph neural networks that efficiently manipulate both node and topological features, with potential applications to chemistry, computational biology and neural algorithmic reasoning.
\newpage
\section*{Broader Impact}

This paper introduced a new methodology for building graph neural networks, conceived independently of a specific application. As graphs constitute a very abstract way to represent data, they have found a lot of different applications \citep{zhou2018graph}. The wide applicability of graph neural networks makes it challenging to foresee how our method will be used and the ethical problems which might occur.

Nevertheless, as we propose to overcome limitations of previous work in learning topological information, our method is likely to be used first and foremost in fields were graph topology is believed to be important. We hope in particular that it can contribute to the fields of quantum chemistry and drug discovery. The good performance obtained on the ZINC dataset is an encouraging sign of the potential of SMP in these fields. Other applications come to mind: material science \citep{schutt2018schnet}, computational biology \citep{gainza2020deciphering}, combinatorial optimization \citep{khalil2017learning,li2018combinatorial,joshi2019efficient} or code generation \citep{brockschmidt2018generative}.

\begin{ack}
Clément Vignac would like to thank the Swiss Data Science Center for supporting him through the PhD fellowship program (grant P18-11).
Andreas Loukas would like to thank the Swiss National Science Foundation for supporting him in the context of the project ``Deep Learning for Graph-Structured Data’' (grant number PZ00P2 179981).
\end{ack}

\bibliographystyle{unsrtnat}
\bibliography{biblio}
\newpage
\appendix


\section{Proof of Theorem 1} \label{proof:theorem1}
Let $f$ be a layer of SMP:
\[
f(\tU, \tY, \mA)[i,:,:] = u(\mU_i, \phi(\{ m(\mU_i, \mU_j, \vy_{ij}) \}_{v_j \in N_i}))
 = u(\mU_i, \phi( \{m(\mU_i, \mU_j, \vy_{ij}) \}_{v_j: \emA[{i,j}] > 0}) )
\]

The action of a permutation $\pi$ on the inputs is defined as 
$f(\pi. (\tU, \tY, \mA)) = f(\pi . \tU, \pi . \tY, \pi . \mA)$.
In order to simplify notation, we will consider $\pi^{-1}$ instead of $\pi$. We have for example $(\pi^{-1}.~ \mA)[{i,j}] = \emA[{\pi_i, \pi_j}]$ and $(\pi^{-1} . \tU)[i,j,k] = \tU[\pi_i, \pi_j, k]$, which can be written as
$$(\pi^{-1} . \tU)[i,:,:] = \mPi ~ \bm{U}_{\pi_i}.$$

As shown next, the theorem's conditions suffice to render SMP equivariant:
\begin{align*}
		f(\pi^{-1} . (\tU, \tY, \mA))[i,:,:] &= u(\mPi ~ \mU_{\pi_i }, ~\phi(\{ m(\mPi ~\mU_{\pi_i},~ \mPi ~\mU_{\pi_j }, ~ \vy_{\pi_i \pi_j})\}_{v_j: A[{\pi_i, \pi_j}] > 0}))  \\
		&= u(\mPi ~ \mU_{\pi_i }, ~\phi(\{m(\mPi ~\mU_{\pi_i}, ~\mPi ~\mU_k,~  \vy_{\pi_i k} )\}_{v_k: A[{\pi_i, k}] > 0}))   && (\pi \text{ bijective) } \\
		&= u(\mPi ~ \mU_{\pi_i }, \mPi ~ \phi (\{m(\mU_{\pi_i}, \mU_k, \vy_{\pi_i k})\}_{v_k:A[{\pi_i, k}] > 0} )) && (\phi \text,~ m \text{ equivariant})\\
		&= \mPi ~ u(\mU_{\pi_i }, \phi 
		(\{m(\mU_{\pi_i}, \mU_k, \vy_{\pi_i k})\}_{ v_k:A[{\pi_i, k}] > 0 } )) && (u \text{ equivariant} )\\
		& = \mPi ~ f(\tU, \tY, \mA)[\pi_i,:,:] \\
		& = (\pi^{-1} . f(\tU, \tY, \mA))[i,:,:],
\end{align*}
which matches the definition of equivariance.

\section{Proof of Theorem~\ref{theorem:representation_power_informal}} 
\label{appendix:proof-turing}
We first present the formal version of the theorem:

\begin{theorem}[Representation power -- formal ]
Consider the class $S$ of simple graphs with $n$ nodes, diameter at most $\diam$ and degree at most $\dmax$. 
We assume that these graphs have respectively $c_X$ and $c_Y$ attributes on the nodes and the edges.
 There exists a permutation equivariant SMP network $f :\R^{n \times n} \mapsto \R^{n \times n \times c}$ of depth at  most $\diam + 1$ and width at most $2 \dmax + c_X + n ~c_Y$ such that, for any two graphs $G$ and $G'$ in $S$ with respective adjacency matrices, node and edge features $\bm{A}, \mX, \tY$ and $\bm{A}', \mX', \tY'$, the following statements hold for every  $v_i \in V$ and $v_j \in V'$:
\begin{itemize}
\item If $G$ and $G'$ are not isomorphic, then for all $\pi \in \mathfrak{S_n}$, $$\bm{\Pi}^T f(\bm{A}, \mX, \tY)[i,:,:] \neq f(\bm{A}', \mX', \tY')[j,:,:].$$ 
\item If $G$ and $G'$ are isomorphic, then for some $\pi \in \mathfrak{S}_n$ independent of $i$ and $j$, $$\bm{\Pi}^T f(\bm{A}, \mX, \tY)[i,:,:] = f(\bm{A}', \mX', \tY')[j,:,:].$$
\end{itemize}
\label{theorem:representation_power}
\end{theorem}

The fact that embeddings produced by isomorphic graphs are permutations one of another is a consequence of equivariance, so we are left to prove the first point. To do so, we will first ignore the features and prove that there is an SMP that maps the initial one-hot encoding of each node to an embedding that allows to reconstruct the adjacency matrix. The case of attributed graphs and the statement of the theorem will then follow easily.

Consider a simple connected graph $G = (V, E)$. For any layer $l\in \mathbb N$ and node $v_i \in V$, we denote by $G_i^{(l)} = (V, E_i^{(l)})$ the graph with node set $V$ and edge set 
$$E_i^{(l)} = \{ (v_p, v_q) \in E, ~d(v_i, v_p) \leq l, ~d(v_i, v_q) \leq l, ~ d(v_i, v_p) + d(v_i, v_q) < 2l \}.$$ 
These edges correspond to the receptive field of node $v_i$ after $l$ layers of message-passing. We denote by $\mA_i^{(l)}$ the adjacency matrix of $G_i^{(l)}$.

\subsection{Warm up: nodes manipulate n x n matrices}

To build intuition, it is useful to first consider the case where $\bm{U}_i$ are $n \times n$ matrices (rather than $n \times c$ as in SMP). In this setting, messages are $n \times n$ matrices as well.
If the initial state of each node $v_i$ is its one-hop neighbourhood ($\mU_i^{(1)} = \mA_i^{(1)}$), then each node can easily recover the full adjacency matrix by updating its internal state as follows:
\begin{align}
    \mU_i^{(l + 1)} = \max_{v_j \in N_i ~\cup~ v_i} \{\mU_j^{(l)}\}, 
    \label{eq:update-equation-matrix}
\end{align}
where the max is taken element-wise.
\begin{lemma}
Recursion~\ref{eq:update-equation-matrix} yields $\mU_i^{(l)}= \mA_i^{(l)}$. 
\label{lemma:warmup}
\end{lemma}

\begin{proof}
We prove the claim by induction. It is true by construction for $l=1$. For the inductive step, suppose that $\mU_i^{(l)}= \mA_i^{(l)}$. Then, 
\begin{align*}
\mU_i^{l + 1}[p, q] = 1 &\iff \exists ~ v_j \in \{N_i \cup v_i\} \quad \text{such that} \quad \mA_j^{(l)}[p, q] = 1 \\ 
&\hspace{-2cm}\iff (v_p, v_q) \in E \text{ and } \exists ~ v_j \in \{N_i \cup v_i\}, ~d(v_j, v_p) \leq l, ~d(v_j, v_q) \leq l, ~ d(v_j, v_p) + d(v_j, v_q) < 2l \\
&\hspace{-2cm}\implies (v_p, v_q) \in E, ~ d(v_i, v_p) \leq l + 1, ~d(v_i, v_q) \leq l + 1, ~ d(v_i, v_p) + d(v_i, v_q) < 2(l+1) \\
&\hspace{-2cm}\implies \mA_i^{(1+1)}[p, q] = 1
\end{align*}
Conversely, if $\mA_i^{(l+1)}[p, q] = 1$, then there exists either a path of length $l$ of the form $(v_i, v_j, \ldots, v_p)$ or $(v_i, v_j, \ldots, v_q)$. This node $v_j$ will satisfy $\mU_j^{(l)}[p, q] = 1$ and thus $\mU_i^{(l+1)}[p, q] = 1$.
\end{proof}
It is an immediate consequence that, for every connected graph of diameter $\diam$, we have $\mU_i^{(\diam)} = \mA$.

\subsection{SMP: nodes manipulate node embeddings}

We now shift to the case of SMP. We will start by proving that we can find an $n \times 2\dmax$ embedding matrix (rather than $n \times n$) that still allows to reconstruct $\mA_i^{(l)}$. For this purpose, we will use the following result:

\begin{lemma}[\citet{maehara1990dimension}] For any simple graph $G = (V, E)$ of $n$ nodes and maximum degree $\dmax$, there exists a unit-norm embedding of the nodes $\mX \in \R^{n \times 2 \dmax}$ such that $$ \forall (v_i,v_j) \in V^2, ~(v_i, v_j) \in E \iff \mX_i \perp \mX_j.$$
\end{lemma} 
In the following we assume the perspective of some node $v_i \in V$. 
Let $\bm{U}_i^{(l)} \in \mathbb{R}^{n \times c_l}$ be the context of $v_i$. Further, write $\bm{u}_j^{(l)} = \bm{U}_i^{(l)}[j,:] \in \mathbb{R}^{c_l}$ to denote the embedding of $v_j$ at layer $l$ from the perspective of $v_i$. Note that, for simplicity, the index $i$ is omitted.

\begin{lemma}
There exists a sequence $(f_l)_{l \geq 1}$ of permutation equivariant SMP layers defining
$ \bm{U}_i^{(l+1)} = f^{(l+1)}(\bm{U}_i^{(l)}, \{\bm{U}_j^{(l)}\}_{j \in N_i}) $
such that $ \bm{u}_j^{(l)} \perp \bm{u}_k^{(l)} \iff (v_j, v_k) \in E_i^{(l)}$ for every layer $l$ and nodes $v_j, v_k \in V$. These functions do not depend on the choice of $v_i \in V$.
\end{lemma} 

\begin{proof}
We use an inductive argument. 
 An initialization (layer $l=1$), we have $\bm{U}_j^{(0)} = \bm{\delta}_j$ for every $v_j$. We need to prove that there exists $\bm{U}_i^{(1)} = f^{(1)}( \mU_i^{(0)}, \{ \bm{U}_j^{(0)} \}_{v_j \in N_i} )$ which satisfies
\[ 
\forall (v_j, v_k) \in V^2, ~ \bm{u}_j^{(1)} \perp \bm{u}_k^{(1)} \iff (v_j, v_k) \in E_i^{(1)}.
\]
Rewritten in matrix form, it is sufficient to show that there exists $\bm{U}_i^{(1)}$ such that $\bm{U}_i^{(1)} (\bm{U}_i^{(1)})^\top = \bm{1} \bm{1}^\top - \bm{A}_i^{(1)}$, with $\bm{1}$ being the all-ones vector. $ \bm{A}_i^{(1)} $ is the adjacency matrix of a star consisting of $v_i$ at the center and all its $d_i$ neighbors at the spokes. Further, it can be constructed in an equivariant manner from the layer's input as follows: 
$$
    \bm{A}_i^{(1)} = \sum_{v_j \in N_i} \bm{\delta}_{i} \bm{\delta}_{j}^\top + \sum_{v_j \in N_i} \bm{\delta}_{j} \bm{\delta}_{i}^\top.  
$$
Since the rank of $ \bm{A}_i^{(1)}$ is at most $d_i$ (there are $d_i$ non-zero rows), the rank of $\bm{1} \bm{1}^\top - \bm{A}_i^{(1)}$ is at most $d_i + 1 \leq 2 d_i \leq 2 d_{\textit{max}}$. It directly follows that there exists a matrix $\bm{U}_i^{(1)}$ of dimension $n \times 2d_{\textit{max}}$ which satisfies $\bm{U}_i^{(1)} (\bm{U}_i^{(1)})^\top = \bm{1} \bm{1}^\top - \bm{A}_i^{(1)}$. Further, as the construction of this matrix is based on the eigendecomposition of $\bm{A}_i^{(1)}$, it is permutation equivariant as desired.   

\emph{Inductive step.} According to the inductive hypothesis, we suppose that:
\[ 
\bm{u}_j^{(l)} \perp \bm{u}_k^{(l)} \iff (v_j, v_k) \in E_i^{(l)} \text{ for all } v_j, v_k \in V
\]
The function $f^{(l+1)}$ builds the embedding $\bm{U}_{i}^{(l+1)}$ from $(\bm{U}_i^{(l)}, \{\bm{U}_j^{(l)}, ~ v_j \in N_i\})$ in three steps: 
\begin{enumerate}[noitemsep, nolistsep]
    \item[Step 1.] Each node $v_j \in N_i$ sends its embedding $\mU_j^{(l)}$ to node $v_i$. This is done using the message function $m^{(l)}$.
    \item[Step 2.] The aggregation function $\phi$ reconstructs the adjacency matrix $\mA_j^{(l)}$ of $G_j^{(l)}$ from  $\bm{U}_j^{(l)}$ for each $v_j \in N_i \cup \{v_i\}$.
    This is done by testing orthogonality conditions, which is a permutation equivariant operation. Then, it computes 
    $\mA_i^{(l+1)}$ as in Lemma~\ref{lemma:warmup} using $\mA_i^{(l+1)} = \max(\{\mA_j^{(l+1)}\}_{v_j \in N_i \cup \{v_i\}})$, with the maximum taken entry-wise.
	\item[Step 3.] The update function $u^{(l)}$ constructs an embedding matrix $\bm{U}_{i}^{(l+1)} \in \R^{n \times 2 \dmax}$ that allows to reconstruct $\mA_i^{(l+1)}$ through orthogonality conditions. The existence of such an embedding is guaranteed by Lemma~\ref{lemma:embedding}. This operation can be performed in a permutation equivariant manner by ensuring that the order of the rows of $\bm{U}_{i}^{(l+1)}$ is identical with that of $\mA_i^{(l+1)}$.
\end{enumerate}
Therefore, the constructed embedding matrix $\bm{U}_{i}^{(l+1)}$ satisfies 
\[ 
\bm{u}_j^{(l+1)} \perp \bm{u}_k^{(l+1)} \iff (v_j, v_k) \in E_i^{(l+1)} \text{ for all } v_j, v_k \in V
\]
and the function $f^{(l+1)}$ is permutation equivariant (as a composition of equivariant functions).
\end{proof}

It is a direct corollary of Lemma~\ref{lemma:embedding} that, when
%
%
the depth is at least as large as the graph diameter, such that $E_i^{(l)} = E$ for all $v_i$ and the width is at least as large as $2 \dmax$,
then there exist a permutation equivariant SMP $f = f^{(L)} \circ \ldots \circ f^{(1)}$ that induces an injective mapping from the adjacency matrix $\mA$ to the local context $\mU_i$ of each node $v_i$. 
As a result, given two graphs $G$ and $G'$, if there are two nodes $v_i \in V$ and $v'_j \in V'$ and a permutation $\pi \in \mathfrak{S}_n$ such that $\mU_i^{(L)} = \mPi^T ~\mU'^{(L)}_j$, then the orthogonality conditions will yield $\mA = \mPi^T ~\mA'~ \mPi$. The contraposition is that if two nodes belong to graphs that are not isomorphic, their embedding will belong to two different equivalence classes (i.e. they will be different even up to permutations).

\subsection{Extension for attributed graphs}

For attributed graphs, the reasoning is very similar: we are looking for a SMP network that maps the attributes to a set of local context matrices such that all the attributes of the graph can be recovered from the context matrix at any node. We treat the case of node and edge attributes separately:

\paragraph{Node attributes}

Using $c_X$ extra channels in SMP is sufficient to create the desired embedding. We recall that the input to the SMP is a local context such that the $i$-th row of $v_i$ contains $[1, \vx_i]$, where $\vx_i$ is the vector of attributes of $v_i$, while the other rows are zero. Ignoring the first entry of this vector (which was used to reconstruct the topology), we propose the following update rule:
\begin{equation}
\mU_i^{(l+1)} = \mU_{j_0}^{(l)} \quad \text{where} \quad j_0 = \text{argmax}(\{|\mU_j^{(l)}| \}_{j \in \{v_i \cup N_i\}})
\label{eq:update_rule}
\end{equation}
where the max is taken element-wise on each entry of the matrix.
This function is simply an extension of the max aggregator that allows to replace the zeros of the local context by non zero values, even if they are negative. Using it, each node can progressively fill the rows corresponding to nodes that are more and more distant. With the assumption that the graph is connected, each node will eventually have access to all node features.

\paragraph{Edge attributes}

As each edge attribute can be seen as a $n \times n$ matrix, edges attributes are handled in a very similar way as the adjacency matrix of unattributed graphs. If nodes could send $n \times n$ matrices as messages, they would be able to recover all the edge features using the previous update rule (equation \ref{eq:update_rule}). However, SMP manipulates embeddings that transform under the action of a permutation as $\pi ~.~ \mU_i = \mPi^T \mU_i$, whereas a $n\times n$ matrix $\mM$ transforms as $\pi ~.~ \mM = \mPi^T \mM ~\mPi$. As a result, we cannot directly pass the incomplete edge features as messages, and we need to embed them into a matrix that permutes in the right way. \\

The construction of a SMP that embeds the input to local contexts that allow to reconstruct an edge feature matrix $\mE$ is the same as for the adjacency matrix, except for one difference: lemma \ref{lemma:embedding}, which was used to embed adjacency matrices into a smaller matrix cannot be used anymore, as it is specific to unweighted graphs. Therefore, we propose another way to embed each matrix $\mE_i^{(l)}$ obtained at node $v_i$ after $l$ message passing layers:
\begin{itemize}
    \item For undirected graphs, $\mE_i^{(l)}$ is symmetric. We can therefore compute its eigendecomposition $\mE_i^{(l)} = \mV \mLambda \mV^T$.
    \item We add a given value $\lambda$ to the diagonal of $\mLambda$ to make sure that all coefficients are non-negative.
    \item We compute the square root matrix $\mU = \mV (\mLambda + \lambda \mI)^{1/2}$. This matrix permutes as desired under the action of a permutation: $\pi~.~ \mU = \mPi^T \mU$. In addition, it allows to reconstruct the matrix $\mE_i^{(l)} = \mU \mU^T$, so that it constitutes a valid embedding for the rest of the proof.
\end{itemize}

Note that the square root matrix permutes as desired, but that it does not compress the representation of $\mE_i^{(l)}$: for each edge features, $n$ additional channels are needed, so that a SMP should have $n \times c_Y$ more channels to be able to reconstruct all edge features.

\subsection{Conclusion}
We have shown that there exists an SMP that satisfies the conditions of the theorem, and specifically, we demonstrated that each layer can be decomposed in a message, aggregation and update functions that should be able to internally manipulate $n \times n$ matrices in order to produce embeddings of size $n \times 2\dmax + c_X + n ~c_Y$.

The main assumption of 
our proof is that the aggregation and update functions can \textit{exactly} compute any function of their input --- this is impossible in practice. 
An extension of our argument to a universal approximation statement would entail substituting the aggregation and update functions by appropriate universal approximators. In particular, the aggregation function manipulates a set of $n \times c$ matrices, which can be represented as a $n \times n \times c$ tensor with some lines zeroed out. Some universal approximators of equivariant functions for these tensors are known \citep{keriven2019universal}, but they have large memory requirements. Therefore, proving that a given parametrization of an SMP can be used to approximately reconstruct the adjacency matrix hinges on the identification of a simple universal approximator of equivariant functions on $n \times n \times c $ tensors.

\section{Proof of Corollary \ref{cor:universality}} \label{proof: universality}

Lemma~\ref{lemma:embedding} proves the existence of an injective mapping from adjacency matrices of simple graphs to features for a set of nodes. Therefore, any permutation equivariant function $h_{\text{eq}}(\bm{A})$ on adjacency matrices can be expressed by an equivariant function on sets
$$ 
h_{\text{eq}}(\bm{A}) = h_{\text{eq}}'(\mU) \quad \text{with} \quad  \mU[i, :] = \vu_i \in \R^{2d_\text{max} + c_X + n ~ c_Y}\quad \forall v_i \in V,
$$ 
as long as the node embeddings $\bm{u}_1, \ldots, \bm{u}_n$ allow the reconstruction of $\bm{A}$, e.g., through orthogonality conditions. 
 It was proven in Theorem~\ref{theorem:representation_power} that, under the corollary's conditions, the local context $\mU_i^{(L)}$ of any node $v_i$ yields an appropriate matrix $\mU$. 
In order to compute $h_\text{eq}$, each node can then rely on the universal $\sigma$ to compute the invariant function:
$$ h''_\text{inv}(\mU, \bm{\mathds{1}}_i) = h'_\text{eq}(\mU)[i, :] = h_\text{eq}(\mA)[i, :] \in \R^{c}.
$$

For invariant functions $h_\textit{in}(\mA) \in \R^c$, it suffices to build the equivariant function $h_\textit{eq}(\mA) = [h_\textit{in}(\mA), \ldots, h_\textit{in}(\mA)] \in \R^{n \times c}$.
Then, if each node $v_i$ computes $h_\text{eq}(\mA)[i,:] = h_\textit{in}(\mA)$, averaging will yield $\frac{1}{n} \sum_{v_i \in V} h_\text{eq}(\mA)[i,:] = h_\textit{in}(\mA)$, as required.

\section{Proof of Proposition \ref{prop:strictly}} \label{proof-prop1}

\paragraph{SMPs are at least as powerful as MPNNs}
We will show by induction that any MPNN can be simulated by an SMP:

\begin{lemma}
For any MPNN mapping initial node features $(\vx_i^{(0)})_{v_i \in V}$ to $(\vx_i^{(L)})_{v_i \in V}$, there is an SMP with the same number of layers such that 
$$ \forall ~v_i \in V, ~\forall ~l \leq L, ~ \tU^{(l)}[i, i, :] = \vx_i^{(l)} \quad \text{and}\quad  \forall j \neq i,~  \tU^{(l)}[i, j, :] = 0. $$
\end{lemma}

\begin{proof}
Consider a graph with node features $( \bm{x}_i^{(0)} )_{v_i \in V}$ and edge features $({ \vy_{ij} })_{(v_i, v_j) \in E}$. 

\emph{Initialization}: 
The context tensor is initialized by mapping the node features on the diagonal of $\tU$: $\tU^{(0)}[i, i, \colon] = \vx_i^{(0)}$. The desired property is then true by construction.

\emph{Inductive step}: Denote by  $( \vx_i^{(l)} )_{v_i \in V} $ the features obtained after $l$ layers of the MPNN. Assume that there is a k-layer SMP such that the local context after $l$ layers contains the same features in its diagonal elements: $\tU^{(l)}[i, i, :] = \vx_i^{(l)}$ and 0 in the other entries. Consider one additional layer of MPNN: 
\[ \vx_i^{(l + 1)}= u(\vx_i^{(l)}, \phi(\{m(\vx_i^{(l)}, \vx_j^{(l)}, \vy_{ij} )\}_{j \in N_i})) \]
and the following SMP layer: 
$$ \mU_i^{(l + 1)} = \textit{diag}(\tilde u(\mU_i^{(l)}, \tilde \phi( \{\tilde m(\bm{1 1}^T~ \mU_i^{(l)}, \bm{1 1}^T ~\mU_j^{(l)}, \vy_{ij} )\}_{j \in N_i}))), $$

where $\tilde m, ~\tilde \phi$ and $\tilde u$ respectively apply the functions $m, \phi$ and $u$ simultaneously on each line of the local context $\mU_i$.
As the only non-zero line of $\mU_i$ is $\mU_i[i, :]$,
$\bm{1 1}^T \mU_i^{(l)}$ replicates the $i$-th line of $\mU_i^{(l)}$ on all the other lines, so that they all share the same content $\vx_i^{(l)}$. After the application of the message passing functions $\tilde m, \tilde \phi$ and $\tilde u$, all the lines of $\mU_i$ therefore contain $\vx_i^{(l+1)}$. 

Finally, the function $\textit{diag}$ extracts the main diagonal of the tensor $\tU$ along the two first axes. Let $\delta_{i, j}$ be the function that is equal to 1 if $i=j$ and $0$, otherwise. We have: $\textit{diag}(\tU)[i, j, :] = \tU[i, j, :] ~ \delta_{i, j}$. Note that this function can equivalently be written as an update function applied separately to each node: $\textit{diag}(\mU_i)[j,:] = \mU_i[j, :] \delta_{i, j} $. We now have $\tU^{(l+1)}[i, i; :] = \vx_i^{l+1}$ and $\tU$ equal to $0$ on all the other entries, so that the induction hypothesis is verified at layer $l + 1$.
\end{proof}
As any MPNN can be computed by an SMP, we conclude that SMPs are at least as powerful as MPNNs.

\paragraph{SMP are strictly more powerful}

To prove that SMPs are strictly more powerful than MPNNs, we use a similar argument to \cite{chen2019equivalence, maron2019provably}:
\begin{lemma}
There is an SMP network which yields different outputs for the two graphs of Fig. \ref{fig:regular}, while any MPNN will view these graphs are isomorphic.
\end{lemma}
\begin{figure}[h]
	\centering
	\includegraphics[width=0.5\textwidth]{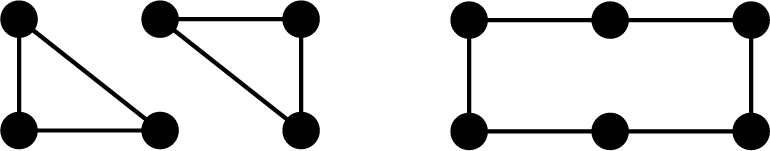}
	\caption{While MPNNs cannot distinguish between two regular graphs such as these ones, SMPs can.}
	\label{fig:regular}
\end{figure}
\begin{proof}
 The two graphs of 
Fig. \ref{fig:regular} are regular, which implies that they cannot be distinguished by the Weisfeiler-Lehman test or by MPNNs without special node features \citep{maron2019provably}. On the contrary, consider an SMP $f$ made of three layers computing $\mU_i^{(l+1)} = \sum_{v_i \in N_i} \mU_i^{(l)}$, followed by the trace of $\mU^{(3)}$ as a a pooling function. As each layer can be written $\mU^{(l+1)} = \mA \mU^{(l)}$ and $\mU^{(0)} = \mI_n$, we have $f(\mA) = \textit{tr}(\mA^3)$. In particular $f(\mA)=2$ for the graph on the left, while  $f(\mA)=0$ on the right.\end{proof}

\section{A more compact representation with graph coloring} \label{appendix:coloring}

In SMP, the initial local context is a one-hot encoding of each node: $\mU_i^{(0)} = \bm{\delta}_i \in \R^n$. When the graph diameter $\diam$ is large compared to the number of layers $L$, the memory requirements of this one-hot encoding can be reduced by
attributing the same identifiers to nodes that are far away from each other. In particular, no node should see twice the same identifier in its $L$-hop neighborhood. To do so, we propose to build a graph $G'$ where all $2L$-hop neighbors of $G$ are connected, and to perform a greedy coloring of $G'$ (Algorithm \ref{algo:coloring}). Although the number of colors $\chi$ used by the greedy coloring might not be optimal, this procedure guarantees that identifiers do not conflict.

\begin{algorithm}
\DontPrintSemicolon 
\KwIn{A graph $G = (V, E)$ with $n$ nodes, $L \in \sN$ (number of layers.)}
\KwOut{A binary matrix $\mU_i^0 \in \R^{n \times \chi}$, where $\chi$ is the number of colors.}
Create the graph $G' = (V, ~ \{(i, j), d(i, j\}) \leq 2L)$ \; 
$\vc \in \R^n \gets \textit{greedy\_coloring}(G')$ \;
\Return{$\textit{one\_hot\_encoding}(\vc)$}\;
\caption{Node coloring}
\label{algo:coloring}
\end{algorithm}

The one-hot encoding of the colors $\mU_i^0 \in \R^{\chi}$ is then used to initialize the local context of $v_i$. The only change in the SMP network is that in order to update the representation that node $i$ has of node $j$, we now update $\mU_i[c_j, :]$ instead of $\mU_i[j, :]$, where $c_j$ is the color associated to node $v_j$.
Note however that the coloring is only useful if the graph has a diameter $\diam > 2 L$. This is usually the case in geometric graphs such as meshes, but often not in scale-free networks.

\section{Proof of Proposition \ref{prop:2}} 
\label{proof:prop2}

We will prove by induction that any Fast SMP layer can be approximated by two blocks of PPGN. It implies that the expressive power of Fast SMP is bounded by that of PPGN.

Recall that a block of PPGN is parameterized as:
\[
\tT^{(l+1)} = m_4(m_3(\tT^{(l)}) \| m_1(\tT^{(l)}) ~@~ m_2(\tT^{(l)})),
\]
where $m_k$ are MLPs acting over the third dimension of $\tT \in \R^{n \times n \times c}$: $\forall (i, j), ~ m_k(\tT)[i, j, :] = m_k(\tT[i, j, :])$. Symbol $\|$ denotes concatenation along the third axis and $@$ matrix multiplication performed in parallel on each channel: $(\tT ~@~ \tT')[:, :, c] = \tT[:, :, c] ~\tT'[:, :, c]$. 

To simplify the presentation, we assume that:
\begin{itemize}[noitemsep, nolistsep]
    \item At each layer $l$, one of the channels of $\tT^{(l)}$ corresponds to the adjacency matrix $\mA$, another contains a matrix full of ones $\bm{1}_n \bm{1}_n^\top$ and a third the identity matrix $\mI_n$, so that each PPGN layer has access at all times to these quantities. These matrices can be computed by the first PPGN layer and then kept throughout the computations using residual connections.
    \item The neural network can compute entry-wise multiplications $\odot$. This computation is not possible in the original model, but it can be approximated by a neural network.
    \item $\tU$ and $\tT$ have only one channel (so that we write them $\mU$ and $\mT$). This hypothesis is not necessary, but it will allow us to manipulate matrices instead of tensors.
\end{itemize}

\paragraph{Initialization}
Initially, we simply use the same input for PPGN as for SMP ($\mU^{(0)} = \mT^{(0)} = \mI_n$). 
\paragraph{Induction}
Assume that at layer $l$ we have $\mU^{(l)} = \mT^{(l)}$. Consider a layer of Fast SMP:
\begin{equation*}
\hspace{-1mm}
\mU_i^{(l+1)} 
= \frac{1}{\davg} \left(\sum_{v_j \in N_i} \hat{\mU}^{(l)}_j ~+~  \hat\mU^{(l)}_i \mW_4^{(l)} \odot \sum_{v_j \in N_i} \hat\mU^{(l)}_j \mW_5^{(l)} \right),
\end{equation*}
where
\[
\hat{\mU}_i^{(l)} = \mU_i^{(l)} ~ \bm{W}_1^{(l)} + \frac{1}{n} \bm{1}_n ~  \bm{1}_n^T  ~ \mU_i^{(l)} ~ \bm{W}_2^{(l)} + \bm{1}_n (\vc^{(l)})^\top + \frac{1}{n} \bm{\mathds{1}}_i \bm{1}^T \mU_i^{(l)}~ \bm{W}_3^{(l)}. 
\]
A first PPGN block can be used to compute $\hat \mU_i^{(l)}$ for each node. This block is parametrized by:
\begin{align*}
    m_1(\mU^{(l)}) &= \frac{1}{n} ~\bm{1_n}~ \bm{1_n}^T,
    \qquad &&m_2(\mU^{(l)}) = \mU^{(l)}, \\
    m_3(\mU^{(l)}) &= \mU^{(l)} ~\mW_1 + 1 \vc^T + (\mI_n ~\odot~ \mU^{(l)}) ~ \mW_3,
    \qquad &&m_4(\hat \mU, ~\tilde\mU) = \hat \mU + \tilde \mU ~\mW_2 + \mI_n ~\odot~ (\tilde\mU ~\mW_3)
\end{align*}

The output of this block exactly corresponds to $\hat \mU^{(l)}$.
Then, a second PPGN block can be used to compute the rest of the Fast SMP layer. It should be parametrized as:
\begin{align*}
    &m_1([\hat\mU^{(l)}]) = \mA ~/~ \bar d,
    \qquad && m_2([\hat\mU^{(l)}]) = \hat\mU^{(l)}, \\
    & m_3([\hat\mU^{(l)}]) = \hat\mU^{(l)}, 
    \qquad &&m_4(\hat\mU^{(l)}, ~\tilde\mU) = \tilde\mU + (\hat\mU^{(l)} ~ \bm{W_4}) ~\odot~ (\tilde\mU ~\bm{W_5})
\end{align*}
By plugging these expressions into the definition of a PPGN block, we obtain that the output of this block corresponds to $\mU^{(l+1)}$ as desired.

\newpage
\section{Comparison between SMP, Provably powerful graph networks and Ring-GNN} \label{appendix:provably-smp}

\begin{figure}[h]
	\makebox[\textwidth][c]{%
	\includegraphics[width=1.00\textwidth]{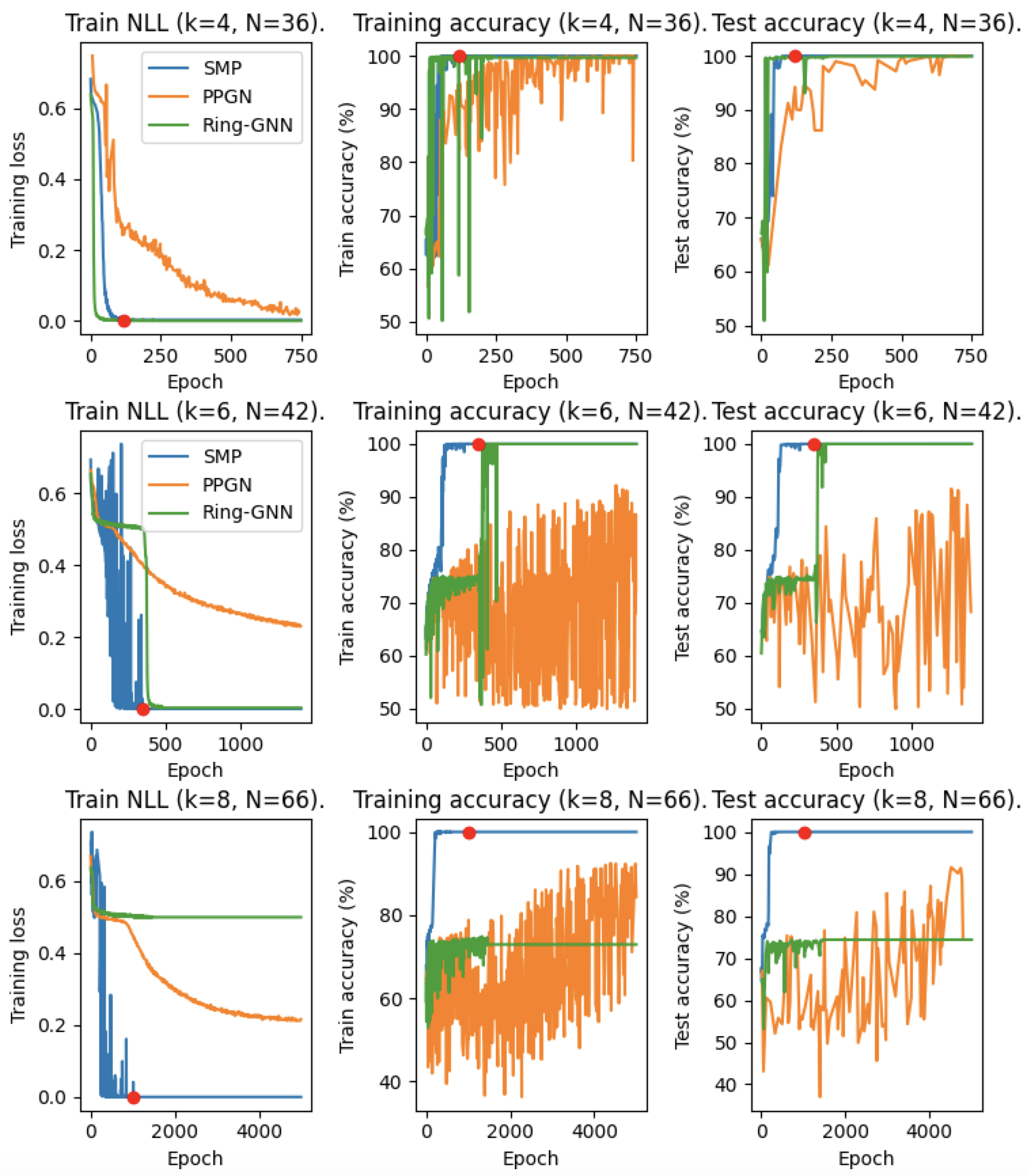}
	}
\caption{Training curves of SMP, PPGN and Ring-GNN for different cycle lengths $k$. NLL stands for negative log-likelihood. Red dots indicate the epoch when SMP training was stopped. The training loss sometimes exhibits peaks of very high value which last one epoch -- they were removed for readability. Provably powerful graph networks are much more difficult to train than SMP: their failure is not due to a poor generalization, but to the difficulty of optimizing them. Ring-GNN works well for small graphs, but we did not manage to train it with the largest graphs (66 or 72 nodes). We attribute this phenomenon to an inductive bias that is less suited to the task. PPGN and SMP training time per epoch are approximately the same, while RING-GNN is between two and three times slower.} \label{fig:smp-provably}
\end{figure}

\end{document}

%% file: math_commands.tex
\usepackage{amsmath,amsfonts,bm}

\def\eqref#1{equation~\ref{#1}}

\def\1{\bm{1}}

\def\vc{{\bm{c}}}

\def\vu{{\bm{u}}}

\def\vx{{\bm{x}}}
\def\vy{{\bm{y}}}

\def\mA{{\bm{A}}}

\def\mE{{\bm{E}}}

\def\mI{{\bm{I}}}

\def\mL{{\bm{L}}}
\def\mLambda{{\bm{Lambda}}}
\def\mM{{\bm{M}}}

\def\mT{{\bm{T}}}
\def\mU{{\bm{U}}}
\def\mV{{\bm{V}}}
\def\mW{{\bm{W}}}
\def\mX{{\bm{X}}}

\def\mPi{{\bm{\Pi}}}

\def\mLambda{{\bm{\Lambda}}}

\DeclareMathAlphabet{\mathsfit}{\encodingdefault}{\sfdefault}{m}{sl}
\SetMathAlphabet{\mathsfit}{bold}{\encodingdefault}{\sfdefault}{bx}{n}
\newcommand{\tens}[1]{\bm{\mathsfit{#1}}}

\def\tT{{\tens{T}}}
\def\tU{{\tens{U}}}

\def\tY{{\tens{Y}}}

\def\sN{{\mathbb{N}}}

\def\emA{{A}}

\newcommand{\R}{\mathbb{R}}



%% file: cycle_results.tex
\begin{subtable}[]{\textwidth}
\caption{Test accuracy on the detection of cycles of various length with 10,000 training samples. (Best seen in color.) Only SMP solves the problem in all configurations.}
\vspace{-0.2cm}
\label{tab:results-cycles}
\resizebox{\textwidth}{!}{%
\begin{tabular}{@{}lrrrrrrrrrrrr@{}}
\toprule
\multicolumn{1}{l|}{Cycle length} & \multicolumn{4}{c|}{4}                                                                                                                                                                                             & \multicolumn{4}{c|}{6}                                                                                                                                                                                            & \multicolumn{4}{c}{8}                                                                                                                                                                                              \\ \midrule
\multicolumn{1}{l|}{Graph size}   & 12                                                  & 20                                                 & 28                                                 & \multicolumn{1}{r|}{36}                            & 20                                                 & 31                                                 & 42                                                 & \multicolumn{1}{r|}{56}                            & 28                                                 & 50                                                 & 66                                                 & 72                                                  \\ \midrule
MPNN                              & \cellcolor[HTML]{9AFF99}{\color[HTML]{000000} 98.5} & \cellcolor[HTML]{FFD88A}93.2                       & \cellcolor[HTML]{FFD88A}91.8                       & \cellcolor[HTML]{EB6767}86.7                       & \cellcolor[HTML]{9AFF99}98.7                       & \cellcolor[HTML]{FFD88A}95.5                       & \cellcolor[HTML]{FFD88A}92.9                       & \cellcolor[HTML]{EB6767}88.0                       & \cellcolor[HTML]{9AFF99}98.0                       & \cellcolor[HTML]{FFD88A}96.3                       & \cellcolor[HTML]{FFD88A}92.5                       & \cellcolor[HTML]{EB6767}89.1                        \\
GIN                               & \cellcolor[HTML]{9AFF99}98.3                        & \cellcolor[HTML]{9AFF99}97.1                       & \cellcolor[HTML]{FFD88A}95.0                       & \cellcolor[HTML]{FFD88A}93.0                       & \cellcolor[HTML]{69E242}99.5                       & \cellcolor[HTML]{9AFF99}97.2                       & \cellcolor[HTML]{FFD88A}95.1                       & \cellcolor[HTML]{FFD88A}92.7                       & \cellcolor[HTML]{9AFF99}98.5                       & \cellcolor[HTML]{9AFF99}98.8                       & \cellcolor[HTML]{FFD88A}90.8                       & \cellcolor[HTML]{FFD88A}92.5                        \\
GIN + degree                      & \cellcolor[HTML]{9AFF99}99.3                        & \cellcolor[HTML]{9AFF99}98.2                       & \cellcolor[HTML]{9AFF99}97.3                       & \cellcolor[HTML]{FFD88A}96.7                       & \cellcolor[HTML]{9AFF99}99.2                       & \cellcolor[HTML]{9AFF99}97.1                       & \cellcolor[HTML]{FFD88A}97.1                       & \cellcolor[HTML]{FFD88A}94.5                       & \cellcolor[HTML]{9AFF99}99.3                       & \cellcolor[HTML]{9AFF99}98.7                       & \cellcolor[HTML]{FFD88A}                           & \cellcolor[HTML]{FFD88A}95.4                        \\
GIN + rand id                     & \cellcolor[HTML]{9AFF99}99.0                        & \cellcolor[HTML]{FFD88A}96.2                       & \cellcolor[HTML]{FFD88A}94.9                       & \cellcolor[HTML]{EB6767}88.3                       & \cellcolor[HTML]{9AFF99}99.0                       & \cellcolor[HTML]{9AFF99}97.8                       & \cellcolor[HTML]{FFD88A}95.1                       & \cellcolor[HTML]{FFD88A}96.1                       & \cellcolor[HTML]{9AFF99}98.6                       & \cellcolor[HTML]{9AFF99}98.0                       & \cellcolor[HTML]{9AFF99}97.2                       & \cellcolor[HTML]{FFD88A}95.3                        \\
RP \cite{murphy2019relational}                               & \cellcolor[HTML]{69E242}100                         & \cellcolor[HTML]{69E242}99.9                       & \cellcolor[HTML]{69E242}99.7                       & \cellcolor[HTML]{9AFF99}97.7                       & \cellcolor[HTML]{9AFF99}99.0                       & \cellcolor[HTML]{9AFF99}97.4                       & \cellcolor[HTML]{FFD88A}92.1                       & \cellcolor[HTML]{EB6767}84.1                       & \cellcolor[HTML]{9AFF99}99.2                       & \cellcolor[HTML]{9AFF99}97.1                       & \cellcolor[HTML]{FFD88A}92.8                       & \cellcolor[HTML]{EB6767}{\color[HTML]{000000} 80.6} \\
PPGN                              & \cellcolor[HTML]{69E242}100                         & \cellcolor[HTML]{69E242}100                        & \cellcolor[HTML]{69E242}100                        & \cellcolor[HTML]{69E242}99.8                       & \cellcolor[HTML]{9AFF99}98.3                       & \cellcolor[HTML]{9AFF99}99.4                       & \cellcolor[HTML]{FFD88A}93.8                       & \cellcolor[HTML]{EB6767}87.1                       & \cellcolor[HTML]{69E242}99.9                       & \cellcolor[HTML]{9AFF99}98.7                       & \cellcolor[HTML]{EB6767}84.4                       & \cellcolor[HTML]{EB6767}76.5                        \\
Ring-GNN                          & \cellcolor[HTML]{69E242}100                         & \cellcolor[HTML]{69E242}99.9                       & \cellcolor[HTML]{69E242}99.9                       & \cellcolor[HTML]{69E242}99.9                       & \cellcolor[HTML]{69E242}100                        & \cellcolor[HTML]{69E242}100                        & \cellcolor[HTML]{69E242}100                        & \cellcolor[HTML]{69E242}100                        & \cellcolor[HTML]{69E242}99.1                       & \cellcolor[HTML]{69E242}99.8                       & \cellcolor[HTML]{EB6767}74.4                       & \cellcolor[HTML]{EB6767}71.4                        \\
SMP                               & \cellcolor[HTML]{69E242}{\color[HTML]{000000} 100}  & \cellcolor[HTML]{69E242}{\color[HTML]{000000} 100} & \cellcolor[HTML]{69E242}{\color[HTML]{000000} 100} & \cellcolor[HTML]{69E242}{\color[HTML]{000000} 100} & \cellcolor[HTML]{69E242}{\color[HTML]{000000} 100} & \cellcolor[HTML]{69E242}{\color[HTML]{000000} 100} & \cellcolor[HTML]{69E242}{\color[HTML]{000000} 100} & \cellcolor[HTML]{69E242}{\color[HTML]{000000} 100} & \cellcolor[HTML]{69E242}{\color[HTML]{000000} 100} & \cellcolor[HTML]{69E242}{\color[HTML]{000000} 100} & \cellcolor[HTML]{69E242}{\color[HTML]{000000} 100} & \cellcolor[HTML]{69E242}{\color[HTML]{000000} 99.9} \\ \bottomrule
\end{tabular}
}
\end{subtable}

%% file: table_generalization.tex
\begin{tabular}{@{}lllllll@{}}
\toprule
Setting      & \multicolumn{3}{c}{In-distribution} & \multicolumn{3}{c}{Out-of-distribution}       \\ \midrule
Cycle length & 4          & 6          & 8         & 4             & 6             & 8             \\
Graph size   & 20         & 31         & 50        & 36            & 56            & 72            \\ \midrule
GIN          & 93.9       & 99.7       & 98.8      & 81.1          & 85.8          & \textbf{88.8} \\
PPGN         &   99.9         &   99.5         &     98.7      &   50.0            &      50.0         &      50.0         \\
Ring-GNN     & 100        & 100        & 99.9      & 50.0          & 50.0          & \textit{OOM}           \\
SMP          & 100        & 99.8       & 99.5      & \textbf{99.8} & \textbf{87.8} & 79.5          \\ \bottomrule
\end{tabular}